\documentclass[twoside]{article}

\usepackage[accepted]{aistats2025_arxiv}
\usepackage{graphicx}
\usepackage{hyperref}
\usepackage{url}
\usepackage{amsmath,amssymb,amsthm}
\usepackage{multicol}
\usepackage{xcolor}
\usepackage{natbib}

\newtheorem{theorem}{Theorem}
\newtheorem{definition}{Definition}
\newtheorem{assumption}{Assumption}
\newtheorem{remark}{Remark}
\newtheorem{lemma}{Lemma}

\begin{document}

\twocolumn[

\aistatstitle{A Generalization Bound for a Family of Implicit Networks}

\aistatsauthor{ Samy Wu Fung \And Benjamin Berkels}

\aistatsaddress{Department of Applied Mathematics and Statistics \\  Colorado School of Mines \And Department of Mathematics \\ RWTH Aachen University} ]

\begin{abstract}
Implicit networks are a class of neural networks whose outputs are defined by the fixed point of a parameterized operator. They have enjoyed success in many applications including natural language processing, image processing, and numerous other applications.
While they have found abundant empirical success, theoretical work on its generalization is still under-explored. In this work, we consider a large family of implicit networks defined parameterized contractive fixed point operators.
We show a generalization bound for this class based on a covering number argument for the Rademacher complexity of these architectures.
\end{abstract}

\section{Introduction}
Implicit networks~\cite{el2021implicit} are a class of neural network architectures whose outputs are defined by the fixed point of a parameterized operator $T_{\psi}(x, d)$. Mathematically, given inputs $d \in \mathbb{R}^m$ and network parameters $\theta  =(\phi,\psi)\in \mathbb{R}^p$, an implicit network outputs a fixed point $x_{\psi,d}^\star$, or often applies one additional layer to the fixed point, i.e., it outputs $P_\phi(x_{\psi,d}^\star)$, where
\begin{equation}
    x_{\psi,d}^\star = T_{\psi}\left(x_{\psi,d}^\star \: ; d\right).
    \label{eq: fixed_point}
\end{equation}
These networks are also known as deep equilibrium networks~\cite{bai2019deep} and in some settings, differentiable optimization layers~\cite{amos2017optnet}, and have been found to be effective for different tasks including natural language processing~\cite{bai2019deep}, image processing~\cite{heaton2023explainable, zhao2023deep, gilton2021deep},  classification~\cite{fung2022jfb}, traffic flow~\cite{mckenzie2024three}, control~\cite{gelphman2025end},
shortest path and knapsack problems~\cite{mckenzie2024differentiating}, and more.
The success of implicit networks is owed to their unique ability to incorporate hard constraints directly into the output representation~\cite{heaton2023explainable,heaton2021feasibility}, enabling precise control over the properties of the learned functions. This flexibility allows for the enforcement of geometric, physical, or structural constraints that are often difficult to impose in traditional neural network architectures.
Additionally, implicit networks have empirically been found to extrapolate; that is, they can be trained on easy tasks (e.g., solving small $9 \times 9$ mazes) and solving larger mazes~\cite{anil2022path, knutson2024on}.
These strengths make implicit networks a powerful tool for applications where both adherence to specific constraints and robust extrapolation are crucial.

While many generalization works exist for deep neural networks~\cite{bartlett2017spectrally, golowich2018size},  and continuous time networks~\cite{marion2024generalization}, to our knowledge, only two prior works exist for implicit networks, one based on Monotone Equilibrium Networks~\cite{pabbaraju2020estimating} which relies on the monotonicity properties to establish bounds, and another which examines over-parameterized implicit networks~\cite{gao2022optimization}. This work extends these efforts by presenting a generalization bound applicable to a broader class of implicit networks, encompassing various architectures beyond those previously studied.

\subsection{Our Contribution}
In this work, we prove a generalization bound for a large family of implicit networks defined by a contractive fixed point operator. These encompass implicit architectures such as Monotone Equilibrium Networks~\cite{winston2020monotone}, single-layer implicit networks~\cite{el2021implicit}, and optimization-based implicit networks~\cite{ heaton2023explainable, Yin2022Learning}.
The core idea is to use a covering number argument based on Dudley's inequality to bound the Rademacher complexity.

\begin{table*}[t]
    \caption{Nomenclature}
    \label{tab:table_nomenclature}
    \vspace{1mm}
    \centering
    \begin{tabular}{l|l}
        \hline
        $m$: input dimension & $n$: output dimension
        \\
        $p=p_\Phi+p_\Psi$: number of parameters & $N$: number of samples
        \\
        $d \in \mathbb{R}^m$: input data & $y \in \mathbb{R}^n$: output data
        \\
        $\mathcal{Y} \subset \mathbb{R}^n$: compact set for variables $y$ & $\mathcal{D} \subset \mathbb{R}^m$: compact set for variables $d$
        \\
        $\ell \colon \mathbb{R}^n\times \mathbb{R}^n \to \mathbb{R}$: loss function &  $\theta =(\phi,\psi) \in \mathbb{R}^p$ network parameters
        \\
        $\mathcal{L}(\theta)$: expected loss & $\hat{\mathcal{L}}(\theta)$: empirical loss
        \\
        $S = \left\{ (d_i, y_i) \right\}_{i=1}^N$: sample set & $D = [ d_1, d_2, \ldots, d_N ] \in \mathbb{R}^{m \times n}$: stacked inputs
        \\
        $C_{\text{params}}$: bound on norm of network parameters & $C_\text{out}$: bound on output of neural network
        \\
        $C_\ell$: bound on loss $\ell$ & $T_{\psi}(x;d)$: fixed point operator
        \\
        $L_x$: Lipschitz constant of $T$ w.r.t $x$ & $L_{\psi}$: Lipschitz constant of $T$ w.r.t $\psi$
        \\
        $L$: Lipschitz constant of fixed point mapping & $C_d$ bound on norm of variables $d \in \mathcal{D}$\\
        $k$: dimension of $x$, i.e. $x_{\psi,d}^\star\in\mathbb{R}^k$ & $P_\phi:\mathbb{R}^k\to\mathbb{R}^n$ final layer\\
        $\Theta=\Phi\times\Psi\subset\mathbb{R}^p$:  admissible parameters&
        $B_r(x):=B_r^X(x):=\{y\in X:\|x-y\|_X<r\}$
    \end{tabular}
\end{table*}

\section{Related Work}
\label{sec: related_works}

\noindent \textbf{Continuous-Time Neural Networks}.
To our knowledge, the first generalization bound for continuous-time neural networks (i.e., neural ODEs~\cite{chen2018neural}) is provided in~\cite{marion2024generalization}, where boundedness and Lipschitz assumptions on the parameters of the neural ODE and Lipschitz continuous loss functions are required. \cite{bleisteingeneralization} extend this work for neural controlled differential equations (NCDEs) which use use a Lipschitz-based argument to obtain a sampling-dependent generalization bound for NCDEs.

\noindent \textbf{Recurrent and Unrolled Neural Networks}.
\cite{zhang2018stabilizing} use a PAC-Bayes approach based on the empirical Rademacher complexity to establish a generalization bound for vanilla RNNs. This work is extended by~\cite{chen2020generalization}, which presents a tighter bound by incorporating the boundedness condition of the hidden state into their analysis.
A special class of RNNs, called unrolled networks~\cite{gregor2010learning, adler2018learned} unroll an optimization algorithm to perform inferences. \cite{behboodi2022compressive} provide a generalization bound for this class of networks in the context of compressive sensing. This work is based on bounding the Rademacher complexity and provides a bound that grows logarithmically in the number of layers.

\noindent \textbf{Implicit Networks}.
Our work is most closely related to~\citet{pabbaraju2020estimating}, which provides generalization bounds for a specific class of implicit networks, known as monotone operator deep equilibrium networks (MON-DEQs)~\cite{winston2020monotone}.
In this work, a generalization bound that is also based on bounding the Rademacher complexity is provided for MON-DEQs (see~\cite[Theorem 3]{pabbaraju2020estimating}) following techniques from~\cite{neyshabur2017pac}.
Rather than focus on MONs, our work provides a bound for arbitrary contractive operators $T_{\psi}$.
Another key difference is that our work provides a bound for the standard empirical loss instead of the empirical margin loss~\cite{pabbaraju2020estimating}.  A more recent work proves generalization bounds for fully trained over-parameterized implicit network~\cite{gao2022optimization}, adapting NTK-based analysis to implicit models. Although our bound applies to a broader family of architectures, the complexity of the bounds makes it less clear how their tightness compares.

\section{Preliminaries}
Consider the supervised learning setting with a sample set of input-output pairs $S = \left\{ (d_i, y_i) \right\}_{i=1}^N$ sampled from unknown distribution $\mathbb{P}_{\text{true}}$.
We define the expected and empirical loss by
\begin{equation}
    \begin{split}
    &\mathcal{L}(\theta) = \mathbb{E}_{(d,y) \sim \mathbb{P}_\text{true}} \left[ \ell\big(P_\phi(x^\star_{\psi, d}), y\big)\right] \quad \text{ and }
    \\
    &\hat{\mathcal{L}}(\theta) = \frac{1}{N} \sum_{i=1}^N \ell\big(P_\phi(x^\star_{\psi, d_i}), y_i\big),
    \end{split}
    \label{eq: expected_empirical_loss}
\end{equation}
respectively, where $x_{\psi,d}^\star$ is a fixed point defined in~\eqref{eq: fixed_point}, $\ell \colon \mathbb{R}^n\times \mathbb{R}^n \to \mathbb{R}$ is a discrepancy function and $\mathbb{P}_\text{true}$ is the joint distribution of input-output pairs. For any parameter $\theta$, the generalization bound we prove in this work will provide an upper bound for $|\mathcal{L}(\theta) - \hat{\mathcal{L}}(\theta)|$.
For ease of presentation, we provide a table for the variable nomenclature that will be used throughout this paper in Table~\ref{tab:table_nomenclature}.
Additionally, $\| \cdot \|$ is taken to be the Euclidean norm henceforth (unless stated otherwise) for brevity.

We denote the set of all hypothesis functions by
\begin{equation}
    \begin{split}
    \mathcal{H} = \Big\{ h \colon \mathbb{R}^m \to \mathbb{R}^n \; : \; &\exists \theta =(\phi,\psi)\in \Theta \; \forall d \in \mathbb{R}^m \\&  h(d) = P_\phi(x^\star_{\psi,d})  \Big\}.
    \end{split}
    \label{eq: calH_def}
\end{equation}
Moreover, let $D = [d_1, d_2, \ldots, d_N] \in \mathbb{R}^{m \times N}$ be the concatenated set inputs from the sample set $S$, and define the set of all possible concatenated outputs as
\begin{equation}
    \mathcal{M} = \{ P_\phi(x^\star_{\psi, d}) \in \mathbb{R}^{n \times N} \; : \;  \theta =(\phi,\psi)\in \Theta \}.
    \label{eq: eq: calM_def}
\end{equation}

\begin{definition}
    An operator $F \colon \mathbb{R}^n \to \mathbb{R}^n$ is contractive with parameter $\kappa \in (0,1)$ if
    \begin{equation}
        \| F(x) - F(y) \| \leq \kappa \| x - y \| \quad \forall \: x,y \in \mathbb{R}^n.
    \end{equation}
\end{definition}

\begin{definition}[\cite{Ta21}, Definition 1.4.1]
    \label{def: covering_number}
    Let $(\mathcal{M}, d)$ be a metric space and $r > 0$. Then the \emph{covering number $\mathcal{N}(\mathcal{M}, d, r)$ of $(\mathcal{M}, d)$ at level} $r$ is the smallest $n\in\mathbb{N}$ such that $\mathcal{M}$ can be covered by $n$ balls of radius $r$.
\end{definition}
\begin{remark}
    When the metric is induced by a norm, we write $\mathcal{N}(\mathcal{M}, \| \cdot \|, r)$.
\end{remark}

\begin{definition}[Empirical Rademacher Complexity]
    \label{def: empirical_rademacher_complexity}
    Consider a class of scalar, real-valued functions $\mathcal{G}$ and a sample set $Z = (z_1, z_2, \ldots, z_N)$. Then, the \emph{empirical Rademacher complexity} is defined as
    \begin{equation}
        R_Z(\mathcal{G}) = \mathbb{E}_{\epsilon} \; \sup_{g \in \mathcal{G}} \frac{1}{N} \sum_{i=1}^N \epsilon_i g(z_i)
    \end{equation}
    where $\epsilon \in \mathbb{R}^N$ is a Rademacher vector, i.e., a vector of independent Rademacher variables $\epsilon_i, i = 1, \ldots, N$ taking the values $\pm 1$ with equal probability.
\end{definition}

\begin{remark}
    Note the Rademacher complexity is with respect to an entire class of functions $\mathcal{G}$ in the definition above.
\end{remark}

\begin{definition}[Subgaussian Process, \cite{foucart2013mathematical}, Definition 8.22]
    Consider a stochastic process $(Z_t)_{t \in \mathcal{T}}$ with the index set $\mathcal{T}$ in a space with pseudometric given by
    \begin{equation}
        d(s,t) = \left( \mathbb{E} |Z_s - Z_t|^2  \right)^{1/2}.
        \label{eq: pseudometric_def}
    \end{equation}
    Let $(Z_t)_{t \in \mathcal{T}}$ be centered, i.e., $\mathbb{E}Z_t=0$ for all $t\in\mathcal{T}$. Then, $(Z_t)_{t \in \mathcal{T}}$ is called \emph{subgaussian} if
    \begin{equation}
        \begin{split}
        &\mathbb{E} \exp( \omega (Z_s - Z_t) ) \leq \exp \left(\frac{\omega^2 d(s,t)^2}{2} \right)
        \\
        &\forall \; s,t \in \mathcal{T}, \omega > 0.
        \end{split}
        \label{eq: subgaussian_process_def}
    \end{equation}
    Moreover, the radius of such a process is defined by
    \begin{equation}
        \Delta(\mathcal{T}) = \sup_{t \in \mathcal{T}} \sqrt{\mathbb{E} |Z_t|^2}.
        \label{eq: radius_Rademacher_definition}
    \end{equation}
\end{definition}
\begin{remark}
    In this work, $\mathcal{T}$ will correspond to all possible hypothesis functions generated by a neural network architecture, and $Z_t$ will correspond to the loss evaluated at a particular choice of parameters for a fixed training set.
\end{remark}

The next two results are the key tools we use to show our generalization bounds.

\begin{theorem}[Dudley's Inequality, Theorem 8.23~\cite{foucart2013mathematical}]
    \label{thm: dudley}
    Let $(Z_t)_{t \in \mathcal{T}}$ be a centered subgaussian process with radius $\Delta(\mathcal{T})$. Then
    \begin{equation}
        \mathbb{E} \; \sup_{t \in \mathcal{T}} Z_t \leq 4\sqrt{2}\int_0^{\Delta(\mathcal{T})/2} \sqrt{ \log(\mathcal{N}(\mathcal{T}, d, r))} dr
    \end{equation}
\end{theorem}

\begin{theorem}[Theorem 26.5,~\cite{shalev2014understanding}]
    Let $\mathcal{H}$ be a family of functions, $\mathcal{S}$ be the training set of $N$ input-output pairs drawn from $\mathbb{P}_{\text{true}}$
    and $\ell$ be a real-valued bounded loss function satisfying $|\ell(h(d), y)| \leq C_{\ell}$ for all  $h \in \mathcal{H}$ and data pairs $(d,y) \in \mathbb{R}^m \times \mathbb{R}^n$. Then with probability at least $1- \delta$ we have, for all $h \in \mathcal{H}$,
    \begin{equation}
        \mathcal{L}(h) \leq \hat{\mathcal{L}}(h) + 2 \mathcal{R}_S (\ell \circ \mathcal{H}) + 4C_{\ell} \sqrt{ \frac{2 \log(4/\delta)}{N}},
    \end{equation}
    where $\ell \circ \mathcal{H}$ is to be understood as
    \begin{equation}
        \ell \circ \mathcal{H}=\{\mathbb{R}^m \times \mathbb{R}^n\to\mathbb{R},(d,y)\mapsto l(h(d),y):h\in\mathcal{H}\}.
        \label{eq:l-circ-H}
    \end{equation}
    \label{thm: expected_loss_error}
\end{theorem}

\begin{remark}
    Dudley's inequality will allow us to bound the Rademacher complexity using a covering number argument.
\end{remark}

\section{Generalization of Implicit Networks}
We state a series of assumptions and lemmas necessary for our result. For ease of presentation, all proofs are provided in the appendix.

\begin{assumption}
    The distribution of the training data $\mathbb{P}_{\text{true}}$ has compact support. We denote these sets by $\mathcal{D} \times \mathcal{Y}$ with $\mathcal{D} \subset \mathbb{R}^m, \mathcal{Y} \subset \mathbb{R}^n$.
    \label{assumption: compact_support}
\end{assumption}

\begin{assumption}
    The loss function $\ell \colon \mathbb{R}^{n} \times \mathbb{R}^n \to \mathbb{R}$ is Lipschitz with constant $L_\ell$.
    \label{assumption: lipschitz_loss}
\end{assumption}

\begin{assumption}[Bounded Network Parameters and Fixed Points]
    \label{assumption: bounded_inputs_outputs}
    \textit{
    The set of parameters and fixed points are bounded. That is,
    \begin{equation}
        \begin{split}
        &\theta \in \Theta=\Phi\times\Psi \; \text{ where }\\&\Phi=\{ \phi \in \mathbb{R}^{p_\Phi} \; : \; \|\phi\| \leq C_{\rm{params},\Phi} \}
        \\
        &\text{and }\Psi=\{ \psi \in \mathbb{R}^{p_\Psi} \; : \; \|\psi\| \leq C_{\rm{params},\Psi}\}\\
        &\text{and }\|x^\star_{\psi,d}\| \leq C_{{\rm out},T}
        \\
        &\text{and } \|P_\phi(x^\star_{\psi,d})\| \leq C_{\rm{out}}\text{ for all }\theta =(\phi,\psi) \in \Theta, d\in\mathcal{D}
        \end{split}
        \end{equation}
        for  $C_{\rm{params},\Phi},C_{\rm{params},\Psi}, C_{{\rm out},T}, C_{\rm{out}} > 0$.}
\end{assumption}

\begin{assumption}[Lipschitz Properties of Fixed Point Operator]
    The operator $\Psi\times\mathbb{R}^n\times \mathcal{D}\to\mathbb{R}^n, (\psi,x,d)\mapsto T_{\psi}(x;d)$
    is contractive in $x$ with constant $L_x \in (0,1)$. Moreover, define $\mathcal{X}^\star = \{x_{\psi,d}^\star : d \in \mathcal{D}, \psi\in\Psi \} $ to be the set of all possible fixed points.
    The operator
    $\Psi\times\mathcal{X}^\star\times \mathcal{D}\to\mathbb{R}^n, (\psi,x,d)\mapsto T_{\psi}(x;d)$ is Lipschitz with respect to $\psi$ whenever $ x \in \mathcal{X}^\star$ with constant $L_{\psi}$%
    . The mapping $\Phi\times \overline{B^{\mathbb{R}^k}_{C_{{\rm out},T}}(0)}\to\mathbb{R}^n:(\phi,x)\mapsto P_\phi(x)$ is Lipschitz in $\phi$ with constant $L_{P,\phi}$ and Lipschitz in $x$ with constant $L_{P,x}$. Here, $B_r^X(x):=\{y\in X:\|x-y\|_X<r\}$ is the ball with radius $r>0$ centered at $x\in X$ for a normed vector space $X$.
    \label{assumption: T_lipschitz_contractive}
\end{assumption}

We note these assumptions can often be enforced in practice.
For example, contractivity of $T_{\psi}(x;d)$ with respect to $x$ can be obtained using, e.g., spectral normalization~\cite{miyato2018spectral}. Moreover, one may enforce the output of the network, $x^\star_{\psi,d}$, by composing it with a projection onto the ball of radius $C_{\rm{out}}$.

Next, we state some results that must be proven in order to obtain our generalization bound. In particular, the core idea is to bound the Rademacher complexity and plug this bound into Theorem~\ref{thm: expected_loss_error}.

\begin{figure*}[t]
    \centering
    \begin{tabular}{ccc}
        \multicolumn{3}{c}{Ellipse Dataset}
        \\
        \hline
        Contractive Layer & MON & Learned Gradient Descent
        \\
        \includegraphics[width=0.3\textwidth]{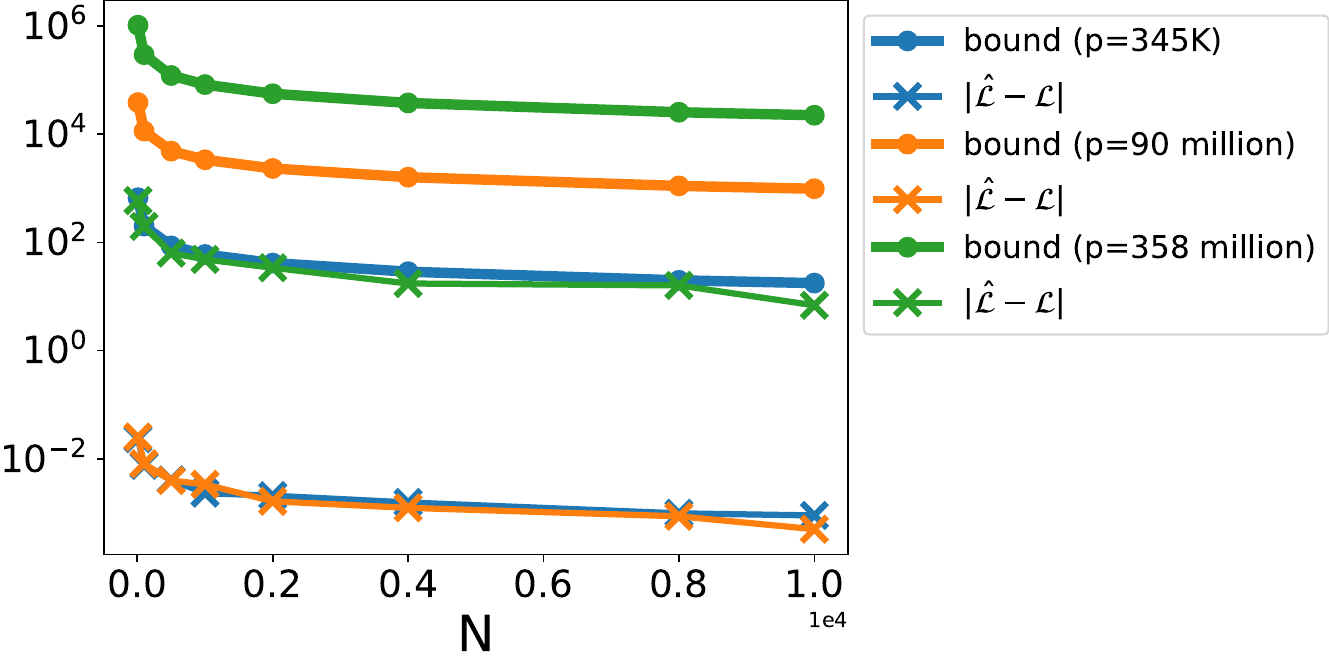}
        &
        \includegraphics[width=0.3\textwidth]{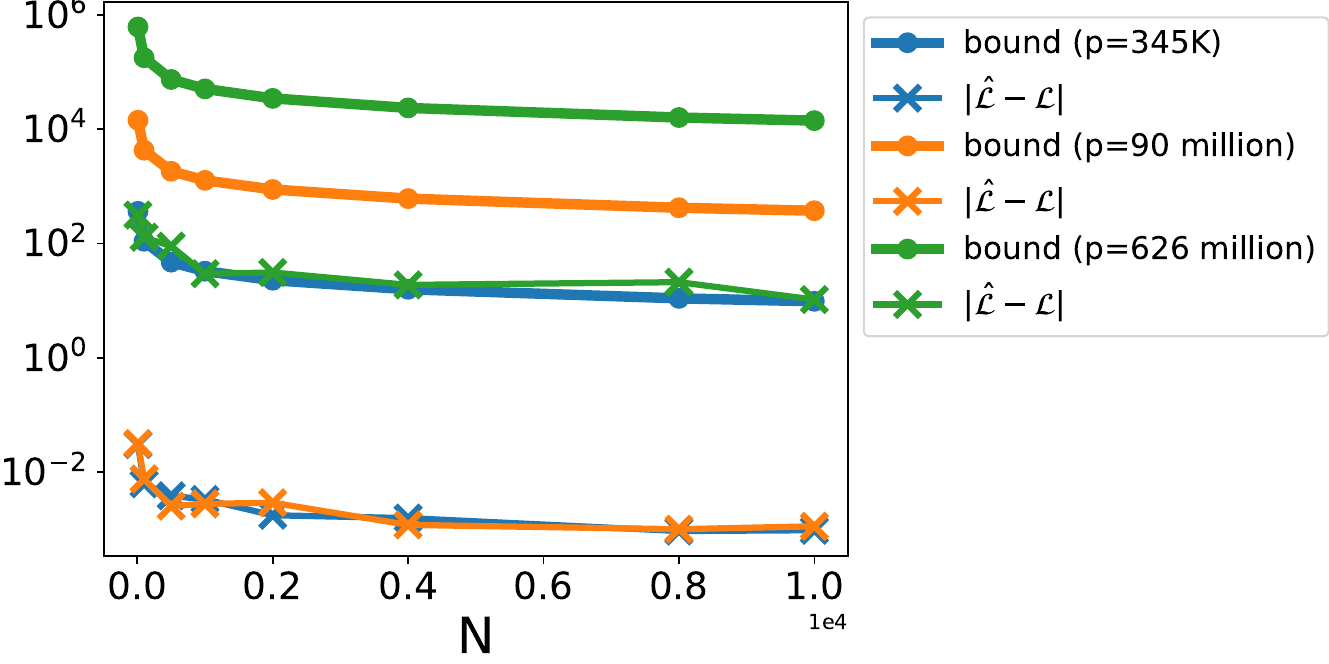}
        &
        \includegraphics[width=0.3\textwidth]{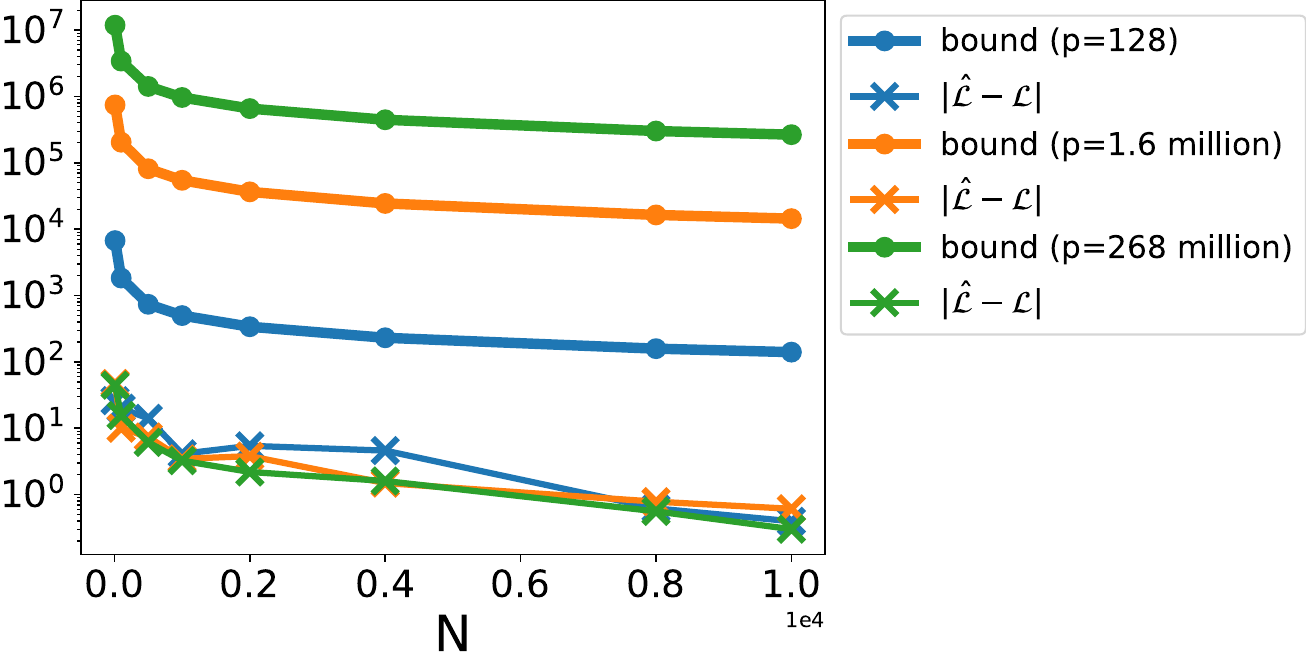}
        \\
        \multicolumn{3}{c}{LoDoPaB Dataset}
        \\
        \hline
        Contractive Layer & MON & Learned Gradient Descent
        \\
        \includegraphics[width=0.3\textwidth]{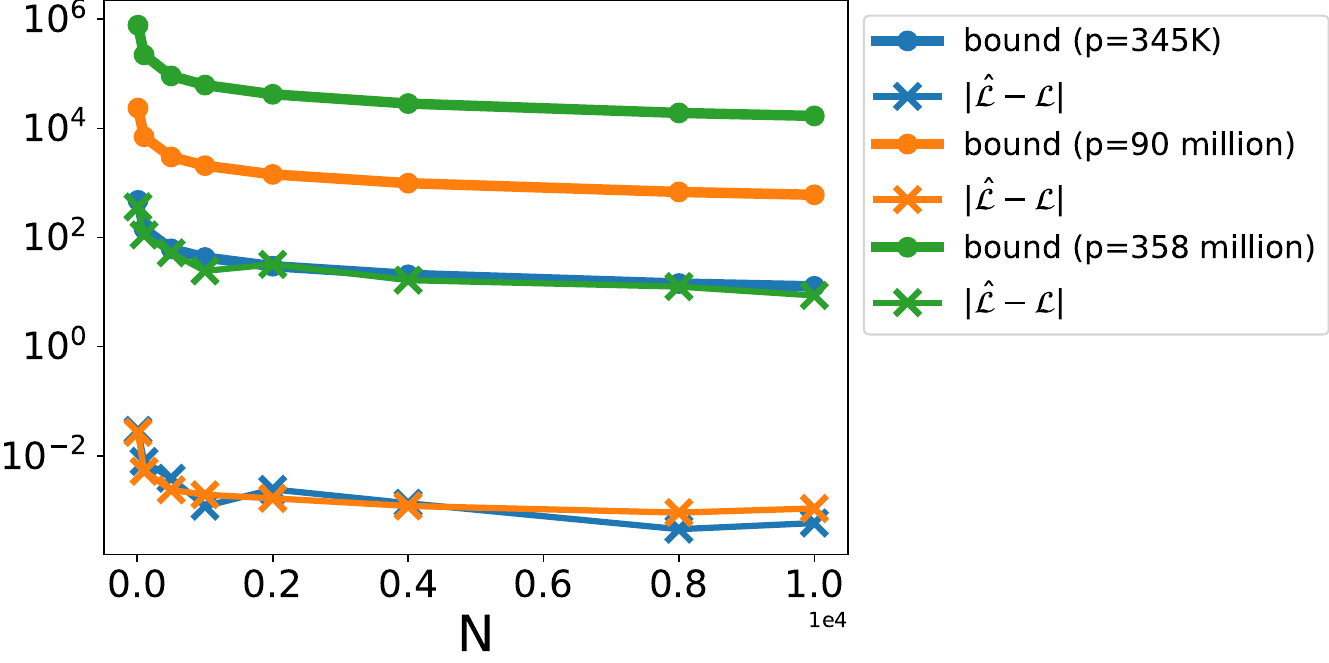}
        &
        \includegraphics[width=0.3\textwidth]{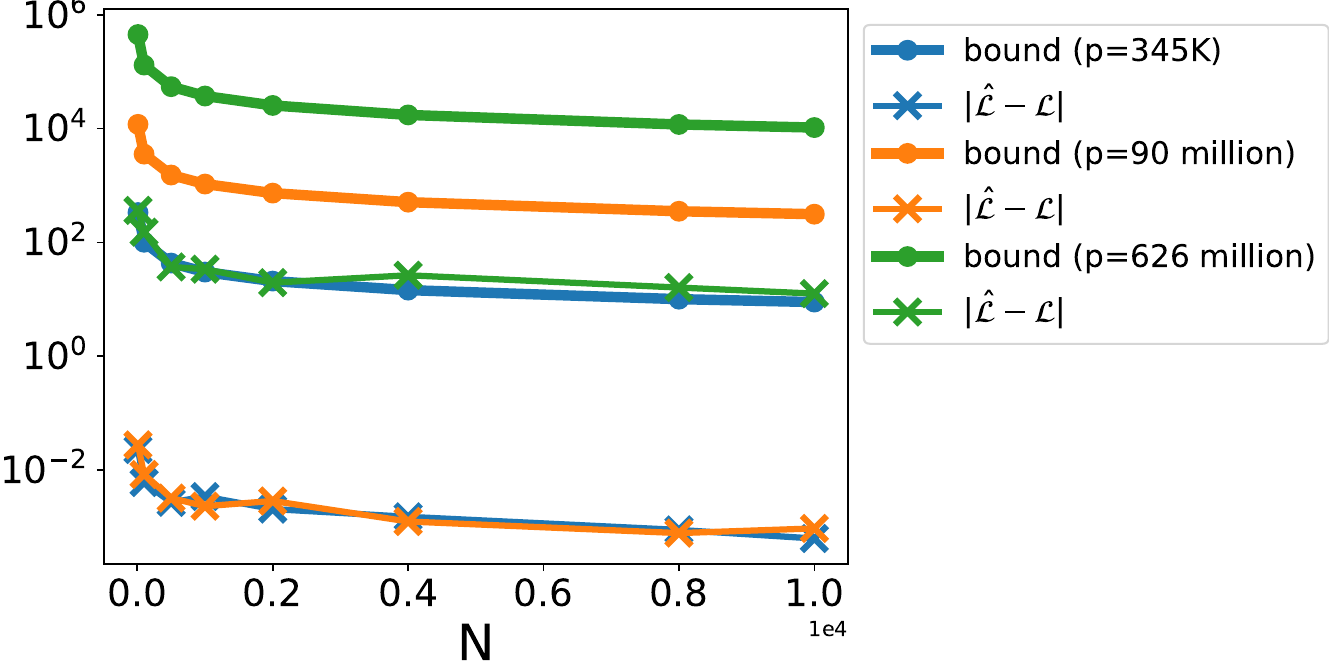}
        &
        \includegraphics[width=0.3\textwidth]{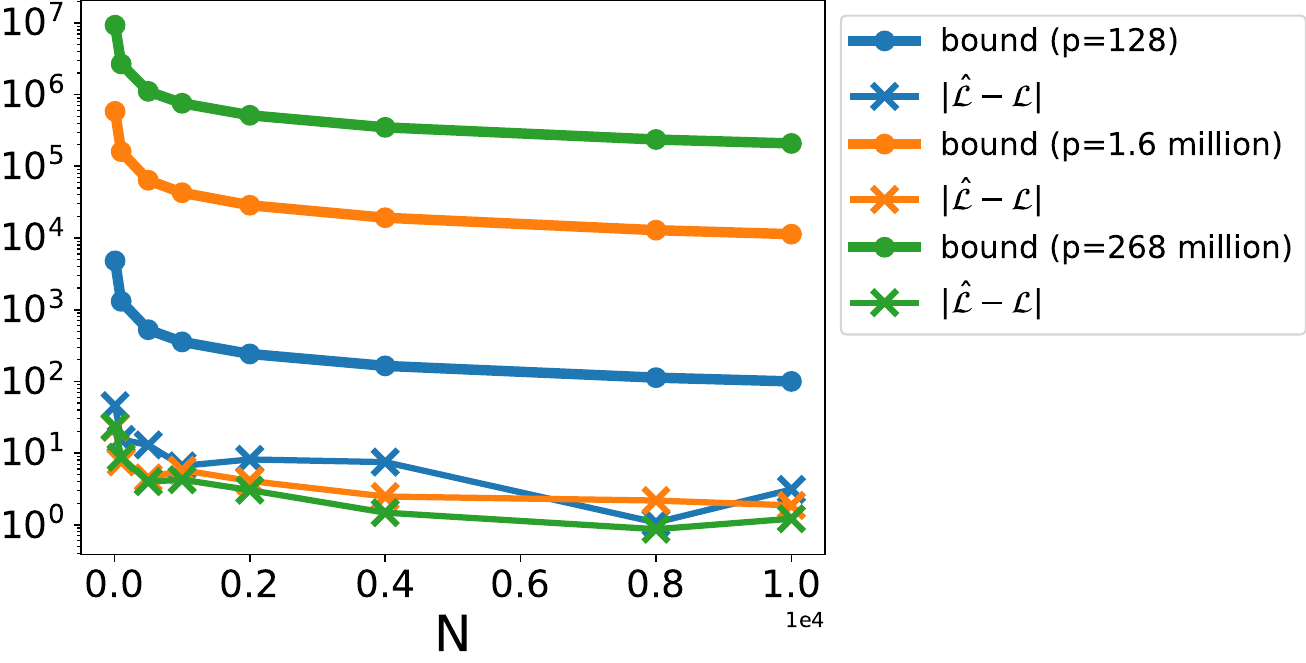}
    \end{tabular}
    \caption{\small Generalization bound (and errors) for different implicit architectures and for different number of parameters $p$. On the x-axis, we have the number of samples (note these are multiplied by $10^4$).}
    \label{fig: results}
\end{figure*}

\newcommand{\lemmaLipschitzFixedpoints}[1]{
Under Assumption~\ref{assumption: T_lipschitz_contractive}, the fixed point mapping $\Psi\times\mathcal{D}\to\mathbb{R}^n,(\psi,d)\mapsto x^\star_{\psi,d}$ is Lipschitz
with respect to $\theta$ with constant $L = \dfrac{L_{\psi}}{1 - L_x}$.
}

\begin{lemma}
\label{lemma: Lipschitz_fixed_points}
\lemmaLipschitzFixedpoints{main}
\end{lemma}

\newcommand{\lemmaLipschitzNetwork}[1]{
Under Assumption~\ref{assumption: T_lipschitz_contractive}, the mapping $\Theta\times\mathcal{D}\to\mathbb{R}^n,((\phi,\psi),d)\mapsto P_\phi(x^\star_{\psi,d})$ is Lipschitz
with respect to $\theta=(\phi,\psi)$ with constant $\hat L=\sqrt{(L_{P,x}L)^2+L_{P,\phi}^2}$.}

\begin{lemma}
\label{lemma: Lipschitz_network}
\lemmaLipschitzNetwork{main}
\end{lemma}

\newcommand{\lemmaBoundRademacher}[1]{
Under Assumptions~\ref{assumption: lipschitz_loss},~\ref{assumption: bounded_inputs_outputs}, and~\ref{assumption: T_lipschitz_contractive}, we have
    \begin{equation}
        \mathcal{R}_S(\ell \circ \mathcal{H}) \leq \frac{8 L_\ell}{N} \int_{0}^{\sqrt{N}C_{\rm{out}}/2} \sqrt{\log \mathcal{N}(\mathcal{M}, \|\cdot\|, r)} \: dr.
    \end{equation}
    where $L_\ell$ is the Lipschitz constant of $\ell$.
}
\begin{lemma}[Bound on Rademacher Complexity]
    \label{lemma: bound_rademacher}
    \lemmaBoundRademacher{main}
\end{lemma}

\newcommand{\lemmaBoundCoveringNumber}[1]{
    Under Assumptions~\ref{assumption: T_lipschitz_contractive} and~\ref{assumption: bounded_inputs_outputs},
    \begin{equation}
        \mathcal{N}(\mathcal{M}, \| \cdot \|, r) \leq \left(1 + \dfrac{2 \hat L C_{\rm{params}}}{r} \right)^p,
    \end{equation}
    where $C_{\rm{params}}=\sqrt{C_{\rm{params},\Phi}^2+C_{\rm{params},\Psi}^2}$.
}

\begin{lemma}[Bound on Covering Number]
    \label{lemma: bound_covering_number}
    \lemmaBoundCoveringNumber{main}
\end{lemma}

\newcommand{\lemmaBoundLoss}[1]{
    Under Assumptions~\ref{assumption: compact_support}, \ref{assumption: lipschitz_loss}, there exists a constant $C_{\ell}>0$ such that
        \begin{equation}
            |\ell(P_\phi(x^\star_{\psi,d}), y)| \leq C_{\ell} \quad \forall \; (d,y) \sim \mathbb{P}_{\text{true}}, \forall \theta =(\phi,\psi) \in \Theta.
        \end{equation}
}
\begin{lemma}[Bound on Loss Function]
    \label{lemma: bound_loss}
    \lemmaBoundLoss{main}
\end{lemma}

\newcommand{\thmGeneralizationBound}[1]{
    Under Assumptions~\ref{assumption: compact_support}--\ref{assumption: T_lipschitz_contractive}, for any $\theta \in \Theta$ and $\delta \in (0,1)$, we have
        \begin{equation}
        \begin{split}
            \mathcal{L}(\theta) &\leq \hat{\mathcal{L}}(\theta)
            \\
            &+ 8L_\ell C_{\rm{out}} \sqrt{\frac{p}{N} \cdot \log \left( e \cdot \left( 1 + \dfrac{4 \hat L C_{\rm{params}}}{\sqrt{N} C_{\rm{out}}} \right)\right)}
            \\
            &+ 4C_{\ell} \sqrt{\dfrac{2 \log(4/\delta)}{N}}
        \end{split}
        \end{equation}
        with probability $1-\delta$.
}

\noindent We now state our main result.
\begin{theorem}[Generalization Bound]
    \label{thm: main_theorem}
    \thmGeneralizationBound{main}
\end{theorem}

A few remarks are in order. First, we follow the generalization bound proof technique from~\cite{behboodi2022compressive}; however, the resulting bound \emph{does not} depend on the ``depth'' of the network as in~\cite{behboodi2022compressive}.
The generalization error is upper bounded by a noise induced term and a complexity term which grows with the square root of the model’s complexity measured by the number of parameters $p$.
Similar to most related works stated in Section~\ref{sec: related_works}, this uses a Lipschitz-based argument and the result is obtained by upper-bounding the Rademacher complexity. Our proof assumes Lipschitzness of the operator $T_{\psi}$ as well as boundedness and a Lipschitz condition on the loss function $\ell$.
We also note that the log term approaches 1 as $N \to \infty$, thus, the bound is $\mathcal{O}(1/\sqrt{N})$.
Finally, the generalization bound in Theorem~\ref{thm: main_theorem} relies boundedness of the network parameters and outputs. This assumption should be understood as a theoretical condition on the admissible hypothesis class rather than a claim about the implicit behavior of unconstrained optimization algorithms. But we emphasize that the admissible parameter radius need not be small. The theoretical results hold for \emph{any} finite bound, and in practice the radius can be chosen sufficiently large so that the constraint is effectively inactive over the course of training.
In this sense, boundedness serves as a technical device required for the covering-number argument rather than a restrictive modeling assumption.

\section{Example Architectures}
We prove Lipschitzness of some standard implicit architectures.
Since the set of training data has compact support, we denote the maximum norm of samples $d$ to be $C_d$. That is, $\| d \| \leq C_d$ for all $d$ in $\mathcal{D}$.

If $k=n$, i.e., the output of the fixed point iteration has the same dimension as the output data, the layer $P_\phi$ is usually not necessary, so it can be chosen as the identity, i.e., $P_\phi(x)=x$. In this case, we get $\Phi=\emptyset$, $p_\Phi=0$, $C_{\rm{params},\Phi}=0$, $L_{P,x}=1$, $L_{P,\phi}=0$ and thus $\hat L=\sqrt{((1\cdot L)^2+0^2)}=L$.

If $k\neq n$, $P_\phi$ is needed to convert the dimension of the fixed point to the one of the output data. Here, it is often sufficient to chose $P_\phi$ as a linear layer, i.e., $\phi=A\in\mathbb{R}^{n\times k}$ and $P_A(x)=Ax$. In this case, $L_{P,x}=C_{\rm{params},\Phi}$ and $L_{P,\phi}=C_{{\rm out},T}$.

\subsection{Single-Layer Contractive Implicit Network}
Perhaps the most standard architecture is a contractive implicit network with fixed point operator of the form
\begin{equation}
    T_{\psi}(x;d) = \sigma(Wx + Ud + b),
    \label{eq: vanilla_T}
\end{equation}
where contractivity w.r.t. $x$ can be guaranteed by using, e.g., spectral normalization~\cite{miyato2018spectral} so that $\|W\| < 1$ and $\sigma\colon \mathbb{R} \to \mathbb{R}$ is a 1-Lipschitz non-linear activation function, e.g., ReLU. Here, the parameters to be learned are $\psi = (W, U, b)$.

\newcommand{\lemmaVanillaTLipschitz}[1]{
        Let $T_{\psi}$ be given by~\eqref{eq: vanilla_T}. Under Assumption~\ref{assumption: bounded_inputs_outputs}, $T_{\psi}$ satisfies Assumption~\ref{assumption: T_lipschitz_contractive} with $L_x = \|W\|$ and $L_{\psi} = \sqrt{C_\text{out}^2 + C_d^2 + 1}$.
}

\begin{lemma}[Lipschitz Property of Contractive Networks]
    \label{lemma: vanilla_T_Lipschitz}
    \lemmaVanillaTLipschitz{main}
\end{lemma}

\subsection{Monotone Equilibrium Network}
We begin with a MON architecture given by the forward-backward iteration described in Section~\ref{sec: related_works}~\cite{pabbaraju2020estimating}:
\begin{equation}
    T_{\psi}(x;d) = \sigma( (I - \alpha(I-W))x + \alpha(Ud + b)),
    \label{eq: mon_iteration_in_W}
\end{equation}
where $\sigma \colon \mathbb{R} \to \mathbb{R}$ is a 1-Lipschitz non-linearity. To ensure $T$ is contractive, one must ensure $I-W$ is strongly monotone; this is achieved by parameterizing $W = (1-m)I - A^\top A + B - B^\top$
for some matrices $A, B$
, and $0 \leq \alpha \leq 2m/\|I-W\|^2$ where $m$ is the strong monotonicity parameter.
Here, the parameters to be learned are $\psi = (A, B, U, b)$.

\newcommand{\lemmaMONLipschitzConstants}[1]{
        Let $T_{\psi}$ be given by~\eqref{eq: mon_iteration_in_W}. Under Assumption~\ref{assumption: bounded_inputs_outputs}, $T_{\psi}$ satisfies Assumption~\ref{assumption: T_lipschitz_contractive} with $L_{\psi} = \alpha \sqrt{4(C_{\text{params}}^2+1)C_{\text{out}}^2+ C_d^2+1}$ and $L_x < 1$.
}

\begin{lemma}[Lipschitz Property of MONs]
    \label{lemma: MON_Lipschitz}
    \lemmaMONLipschitzConstants{main}
\end{lemma}

\subsection{Gradient-Descent Based Network}
In inverse problems, it is of interest to learn a  regularization operator $R$ ~\cite{alberti2021learning} arising from a problem of the form
\begin{equation}
    \min_x \frac{1}{2}\|Ax - b\|^2 + \frac{1}{2}\|Rx\|^2
\end{equation}
In this case, one may learn a gradient descent scheme where the fixed point operator is given by
\begin{equation}
    T_{\psi}(x;d) = x - \alpha(A^\top (Ax - d) + R^\top R x)
    \label{eq: learned_gradient_descent}
\end{equation}
Here, $A$ is given and $\psi = R \in \mathbb{R}^{n_1 \times n}$ is the regularization operator to be learned and $\alpha \in \left(0, \frac{2}{\lambda_{\text{max}}(A^\top A + R^\top R)}\right)$~\cite{ryu2022large}.

\newcommand{\lemmaLearnedGradDescentRegularizer}[1]{
       Let $T_{\psi}$ be given by~\eqref{eq: learned_gradient_descent}. Under Assumption~\ref{assumption: bounded_inputs_outputs}, $T_{\psi}$ satisfies Assumption~\ref{assumption: T_lipschitz_contractive} with $L_{\psi} = 2 \alpha C_{\text{out}} C_{\text{params},\Psi}$ and $L_x < 1$.
}

\begin{lemma}[Lipschitz Property of Learned Gradient Descent]
    \label{lemma: learnedGD}
    \lemmaLearnedGradDescentRegularizer{main}
\end{lemma}

\begin{figure}
    \centering
    Bounds for MNIST Dataset
    \\
    \vspace{1mm}
    \includegraphics[width=0.4\textwidth]{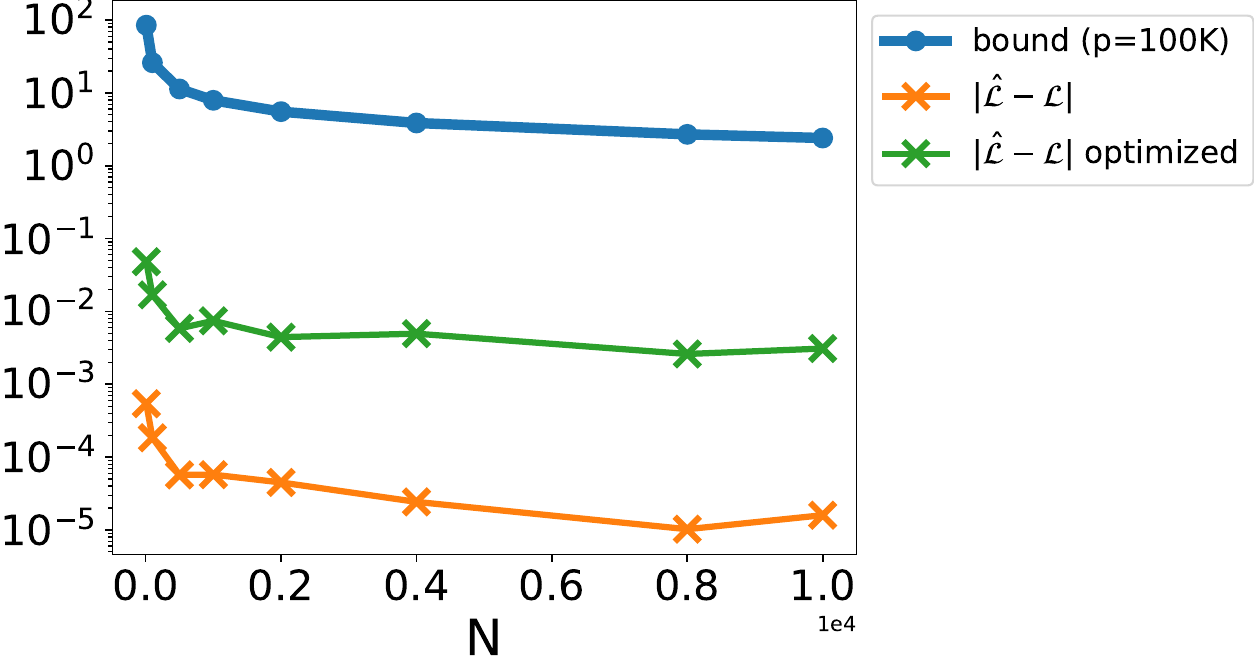}
    \caption{\small Generalization bound for the MNIST dataset. Here, the blue line represents the generalization bound, and the green and orange line represent the generalization error for an optimized and random set of weights, respectively}
    \label{fig: results_mnist}
\end{figure}

\section{Experiments}

\subsection{Data Setup}
We use a set of different datasets to test our generalization bounds.
The first dataset is the LoDoPaB dataset \cite{leuschner2021lodopab, heaton2022wasserstein} and the MNIST dataset~\cite{lecun1998mnist}.
We choose the MNIST dataset because it is a standard benchmark for image classification tasks, which allows us to evaluate our framework using the cross-entropy loss composed with the softmax activation. Moreover, we choose the LoDoPab dataset because it is widely adopted in \emph{learning-to-optimize} studies, where the network architecture is designed to resemble an optimization algorithm such as our learned gradient descent model.

For the LoDoPaB data, CT measurements are simulated with a parallel beam geometry with a sparse-angle setup of only $30$ angles and $183$ projection beams, resulting in 5,490 equations and 16,384 unknowns. In these experiments, $1.5\%$ Gaussian noise corresponding to each individual beam measurement is added to the measurements.
The images have a resolution of $128 \times 128$ pixels.
The ellipse training and test sets contain 11,000 total sample pairs.
The LoDoPaB training and test sets contain 22,000 sample pairs, respectively.

\subsection{Experimental Setup and Results}
For all three architectures, we estimate the required constants; these include $C_{\text{out}},  C_{\text{params}}$, and $C_d$. To estimate $C_\text{out}$, we compute $x^\star_d$ for all $d$ in the dataset over 100 randomly generate samples of $\theta$. The constant \(L_x\) is not estimated but deliberately chosen to be 0.9 to ensure the network’s contractivity. This is achieved through \emph{spectral normalization}~\cite{miyato2018spectral}, which enforces a bound on the spectral norm of the weight matrices using the power iteration method~\cite{golub2013matrix}.

To compute $C_d$, we find the maximum norm of each of the samples provided. To find $C_\text{params}$, we normalize each of the weights to have norm 1; specifically, we bound $\|A\|, \|B\|, \|U\|, \|b\| \leq 1$ so that $C_{\text{params}} = \sqrt{4}$ for MON. We also bound $\|W\|, \|U\|, \|b\| \leq 1$ so that $C_\text{params} \leq \sqrt{3}$ for the single layer contractive network, and $\|R\| \leq 1$ for the learned gradient descent network.
For the CT dataset, we choose $P_\phi(x) = x$, and for the MNIST dataset, we choose a linear final mapping layer given by $P_\phi(x) = \phi x$, where $\phi \in \mathbb{R}^{n \times k}$. Finally, we choose the ReLU activation function for all architectures for the CT dataset and the leaky ReLU activation for the MNIST data, which are commonly used in their respective tasks. Note both of these are non-expansive.

Moreover, we choose the $l_1$-norm for $\ell(\cdot, \cdot)$ and the cross-entropy function for the MNIST dataset, and estimate their maximum bound in a similar manner to $C_\text{out}$.
We note these choices of $\ell$ have Lipschitz constant $L_\ell = 1$ for $\ell = \|\cdot\|_1$ and $L_\ell = 2$ for the cross entropy function composed with softmax (see Lemma~\ref{lemma: Lipschitz_CE}) in the appendix.
We set $\delta = 10^{-2}$ for all experiments.

Figure~\ref{fig: results} shows the generalization bound over different sample sizes $N$ and different number of parameters $p$.
As expected, the bound decreases for smaller $N$ and $p$.
We also observe similar results for the MNIST dataset in Figure~\ref{fig: results_mnist}, where we show the bound (in blue), the generalization error of a randomly initialized network (orange), and the generalization error of an optimized implicit network. Here, the $x$ axis represents the number of samples used to estimate both, the bound and the errors. As expected, all bounds and errors behave like $\mathcal{O}(1/\sqrt{N})$. Moreover, we observe that the bounds are not particularly tight; however, this is to be expected for bounds that are especially general such as the one presented in this work, which can be applied to a large family of implicit networks. We also include a robustness experiment in Appendix~\ref{sec:additions_experiments}.

Finally, we note that in Figure~\ref{fig: results_mnist}, the optimized weights lead to a larger generalization discrepancy; this is because for randomly initialized neural networks, the predicted probabilities are roughly uniform, leading to the same prediction regardless of batchsize.

\section{Conclusion}
In this work, we derive a generalization bound for contractive implicit neural networks. Our bound employs a covering number argument to bound the Rademacher complexity of these architectures. Our work relies on Lipschitz assumptions of the fixed point operator as well as boundedness of the fixed point outputs.
We note that while we only consider the implicit network, one can extend our bounds when embedding these as a differentiable layer in a larger architecture~\cite{amos2017optnet}. Some recent work that may aid in the further study of generalization of implicit networks reveals connections between implicit networks and Gaussian processes~\citep{gao2023wide} and global convergence of implicit networks~\citep{gao2022global}.
We focus on providing a foundational generalization bound for a broad class of implicit neural networks defined via contractive fixed-point operators. While comparative studies with depth-equivalent explicit (unrolled) networks are of interest, a principled comparison would require matched assumptions and corresponding explicit-network bounds and is therefore left for future work.
Finally, while contractivity can be enforced through techniques such as spectral normalization and has proven effective in many settings~\cite{heaton2023explainable, mckenzie2024three}, such enforcement can potentially limit the expressive capacity of these networks. This trade-off helps explain why many successful implementations of implicit networks do not enforce strict global contractivity during training in practice. Instead, they rely on alternative stability mechanisms such as careful initialization and Jacobian regularization. Understanding how these stability mechanisms can play a role for a better generalization bound remains an important direction for future work.

\section*{Acknowledgments}
Part of this research was performed while the authors were visiting the Institute for Pure and Applied Mathematics (IPAM), which is supported by the National Science Foundation, award number DMS-1925919. Samy Wu Fung was also partially funded by the National Science Foundation, award number DMS-2309810.
Benjamin Berkels was supported by the German research foundation (DFG) within the Collaborative Research Centre SFB 1481 ``Sparsity and Singular Structures'' (Project ID 442047500, Project A08).

\bibliographystyle{apalike}
\bibliography{refs}

\onecolumn
\newpage
\begin{center}
\large{\textbf{A Generalization Bound for a Family of Implicit Networks: Supplementary Materials}}
\end{center}
\setcounter{section}{0}
\renewcommand{\thesection}{\Alph{section}}
\section{Proofs}
For ease of presentation, we restate the results before their proofs.
\subsection{Proof of Lemma~\ref{lemma: Lipschitz_fixed_points}}
\noindent
\textbf{Lemma \ref{lemma: Lipschitz_fixed_points}.}
\textit{\lemmaLipschitzFixedpoints{app}}

\noindent Before beginning our proof, we note that since we treat the network parameters as inputs, we write $x^\star_{\psi,d}$ as $x^\star_d(\psi)$ in this proof for ease of presentation.
\begin{proof}
    Fix an input $d \in \mathcal{D}$. For any $\psi_1, \psi_2 \in \Psi$, we have
    \begin{align}
        \|x_d^\star(\psi_1) - x_d^\star(\psi_2)\|
        &=
        \|T_{\psi_1}(x_d^\star(\psi_1) \:; d) -  T_{\psi_2}\left(x_d^\star(\psi_2) \:; d\right) \| \label{eq: temp_sum1}
        \\
        &=
        \|T_{\psi_1}(x_d^\star(\psi_1) \:; d) -  T_{\psi_1}\left(x_d^\star(\psi_2) \:; d\right)
        +
        T_{\psi_1}\left(x_d^\star(\psi_2) \:; d\right)
        -
        T_{\psi_2}\left(x_d^\star(\psi_2) \:; d\right) \|
        \\
        &\leq
        \|T_{\psi_1}(x_d^\star(\psi_1) \:; d) -  T_{\psi_1}\left(x_d^\star(\psi_2) \:; d\right)\| +
        \| T_{\psi_1}\left(x_d^\star(\psi_2) \:; d\right)
        -
        T_{\psi_2}\left(x_d^\star(\psi_2) \:; d\right) \|
        \\
        &\leq
        L_x \|x_d^\star(\psi_2) - x_d^\star(\psi_1)\| + L_{\psi} \| \psi_2 - \psi_1\|. \label{eq: temp_sum2}
    \end{align}
    Since the first term in~\eqref{eq: temp_sum2} contains $\|x_{d}^\star(\psi_1) - x_{d}^\star(\psi_2)\|_2$, we can expand this term again following the same procedure~\eqref{eq: temp_sum1}--\eqref{eq: temp_sum2}. Expanding $N$ times, we obtain
    \begin{align}
        \|x_{d}^\star(\psi_1) - x_{d}^\star(\psi_2)\| &\leq L_x \|x_d^\star(\psi_2) - x_d^\star(\psi_1)\| + L_{\psi} \| \psi_2 - \psi_1\|\\
        &\leq L_x \left(L_x \|x_d^\star(\psi_2) - x_d^\star(\psi_1)\| + L_{\psi} \| \psi_2 - \psi_1\|\right) + L_{\psi} \| \psi_2 - \psi_1\|\\
        &\qquad\vdots\\
        &\leq (L_x)^{N+1} \| x^\star_d(\psi_2) - x^\star_d(\psi_1)\| + L_{\psi} \sum_{k=0}^N (L_x)^{k}\| \psi_2 - \psi_1\|.
        \label{eq: temp_sumN}
    \end{align}
    Taking the limit as $N \to \infty$ and using geometric series (since $L_x < 1$), we obtain
    \begin{equation}
        \|x_{d}^\star(\psi_1) - x_{d}^\star(\psi_2)\| \leq \frac{L_{\psi}}{1 - L_x} \|\psi_2 - \psi_1\|,
    \end{equation}
    where we observe that the second term in~\eqref{eq: temp_sumN} goes to 0.
\end{proof}

\subsection{Proof of Lemma~\ref{lemma: Lipschitz_network}}
\noindent
\textbf{Lemma \ref{lemma: Lipschitz_network} .}
\textit{\lemmaLipschitzNetwork{app}}

\begin{proof}
Fix an input $d \in \mathcal{D}$. For any $\theta_1,\theta_2 \in \Theta$, we have $\theta_i=(\phi_i,\psi_i)$ and get
\begin{align}
    \| P_{\phi_1}(x^\star_{\psi_1,d}) - P_{\phi_2} (x^\star_{\psi_2,d}) \|
    &= \|  P_{\phi_1}(x^\star_{\psi_1,d}) -  P_{\phi_1}( x^\star_{\psi_2,d}) + P_{\phi_1}(x^\star_{\psi_2,d}) - P_{\phi_2} (x^\star_{\psi_2,d}) \|
    \\
    &\leq
    \|  P_{\phi_1}(x^\star_{\psi_1,d}) -  P_{\phi_1}( x^\star_{\psi_2,d}) \| + \| P_{\phi_1}(x^\star_{\psi_2,d}) - P_{\phi_2} (x^\star_{\psi_2,d}) \|
    \\
    ~^{\text{ using}}_{\text{Assumption~\ref{assumption: T_lipschitz_contractive}}}&\leq
    L_{P,x}\| x^\star_{\psi_1,d} - x^\star_{\psi_2,d} \| + L_{P,\phi}\|\phi_1 - \phi_2\|
    \\
    ~^{\text{ using}}_{\text{Lemma~\ref{lemma: Lipschitz_fixed_points}}}&\leq L_{P,x} L \|\psi_1 - \psi_2\| + L_{P,\phi}\|\phi_1 - \phi_2\|
    \\
    &=(L_{P,x} L ,L_{P,\phi})\cdot(\|\psi_1 - \psi_2\|,\|\phi_1 - \phi_2\|)
    \\
    &\leq\sqrt{(L_{P,x} L)^2 +(L_{P,\phi})^2}\sqrt{\|\psi_1 - \psi_2\|^2+\|\phi_1 - \phi_2\|^2}
    \\
    &=\sqrt{(L_{P,x} L)^2 +(L_{P,\phi})^2}\|\theta_1 - \theta_2\|
\end{align}
\end{proof}
\subsection{Proof of Lemma~\ref{lemma: bound_rademacher}}
\noindent
\textbf{Lemma \ref{lemma: bound_rademacher}.}
\textit{\lemmaBoundRademacher{app}}
\begin{proof}

    This proof is structured as follows.
    \textbf{Step 1}, we show that the pseudometric defined in~\eqref{eq: pseudometric_def} associated with the Rademacher process is simply the l2 norm of the difference of two random variables.
    \textbf{Step 2}, we bound the Rademacher complexity (Definition~\ref{def: empirical_rademacher_complexity}) by another Rademacher process that only depends on the hypothesis function. \textbf{Step 3}, we bound the radius of the Rademacher process.
    \textbf{Step 4}, we obtain a bound on the Rademacher complexity using a covering number argument by plugging results from Steps 1 and 2 into Dudley's inequality.

    \textbf{Step 1}.
    We introduce a specific subgaussian process and a representation for its pseudometric.
    Let $\mathcal{T}$ be an index set and $z:\mathcal{T}\rightarrow\mathbb{R}^N$ a function. Furthermore, let $\epsilon \in \mathbb{R}^N$ be
    a Rademacher vector, i.e., a random vector with entries $\pm 1$. Then, the stochastic process $(Z_t)_{t \in \mathcal{T}}$ given by
    \begin{equation}
        Z_t := \sum_{i=1}^N \epsilon_i (z(t))_i = \epsilon^\top z(t),
    \end{equation}
    is a Rademacher process, cf. \cite[p. 225]{foucart2013mathematical}.
    Moreover, this process is subgaussian and we get
    \begin{equation}
        \mathbb{E} |Z_s - Z_t|^2 = \mathbb{E} \left| \epsilon^\top (z(s) - z(t)) \right|^2 = \|z(s) - z(t)\|^2.
    \end{equation}
    Thus, we have that
    \begin{equation}
        d(Z_s, Z_t) =  \left(\mathbb{E} |Z_s - Z_t|^2\right)^{1/2} = \| z(s) - z(t) \|.
        \label{eq: pseudometric_is_l2}
    \end{equation}

    \textbf{Step 2}. From Definition~\ref{def: empirical_rademacher_complexity} and \eqref{eq:l-circ-H}, we have that
    \begin{align}
        R_S(\mathcal{\ell \circ \mathcal{H}})
        &= \mathbb{E}_{\epsilon} \; \sup_{h \in \mathcal{H}} \frac{1}{N} \sum_{i=1}^N \epsilon_i \ell\big(h(d_i), y_i\big)
        \label{eq:rademacher_representation}
    \end{align}
    with $\epsilon \in \mathbb{R}^N$ a Rademacher vector.

Now, we need the so-called vector-valued contraction lemma \cite{maurer2016vector}. To point out that it is applicable in our case, we cite it here in full, just slightly adjusting it to our notation:

\textbf{Contraction lemma~\cite[Corollary 4]{maurer2016vector}} Let $\mathcal{X}$ be any set, $(x_1,\ldots,x_n)\in \mathcal{X}^n$, let $F$ be a class of functions $f:\mathcal{X}\to \ell_2$ and let $h_i:\ell_2\to \mathbb{R}$ be Lipschitz continuous functions with constant $L$.
Here, $\ell_2$ denotes the space of square summable sequences.
Then,
\[
\mathbb{E}_\epsilon\sup_{f\in F}\sum_{i} \epsilon_i\, h_i\!\big(f(x_i)\big)
\le \sqrt{2}\,L\,\mathbb{E}_{\tilde\epsilon}\sup_{f\in F}\sum_{i,k}\tilde\epsilon_{ik}\, (f(x_i))_k,
\]
where $\epsilon\in\mathbb{R}^n$ is a Rademacher vector, $\tilde\epsilon_{ik}$ is an independent doubly indexed Rademacher sequence and $(f(x_i))_k$ is the $k$-th component of $f(x_i)$.

    Since $\ell$ is Lipschitz by Assumption~\ref{assumption: lipschitz_loss}, we have that by the contraction lemma, choosing $h_i$ in the contraction lemma as $l(\cdot,y_i)$ and embedding $\mathbb{R}^m$ in $\ell_2$,
    \begin{equation}
        \mathbb{E}_{\epsilon} \sup_{h \in \mathcal{H}} \; \sum_{i=1}^N \epsilon_i\ell\big(h(d_i), y_i\big)
        \leq
        \sqrt{2} L_\ell  \mathbb{E}_{\tilde{\epsilon}} \; \sup_{h \in \mathcal{H}}  \sum_{i=1}^N \sum_{j=1}^m \tilde{\epsilon}_{ij} \big(h(d_i)\big)_j,
        \label{eq: bound1_rademacher}
    \end{equation}
    where $\tilde{\epsilon} \in \mathbb{R}^{N \times m}$ is a Rademacher matrix and $L_\ell$ is the Lipschitz constant of $\ell$%
    . Thus, we can bound the Rademacher complexity by plugging~\eqref{eq: bound1_rademacher} into~\eqref{eq:rademacher_representation} to get
    \begin{equation}
        \begin{split}
        R_S(\mathcal{\ell \circ \mathcal{H}})
        &\leq
        \frac{\sqrt{2} L_\ell}{N} \: \mathbb{E}_{\tilde{\epsilon}} \; \sup_{h \in \mathcal{H}}  \sum_{i=1}^N \sum_{j=1}^m \tilde{\epsilon}_{ij} \big(h(d_i)\big)_j
        =
        \frac{\sqrt{2} L_\ell}{N} \: \mathbb{E}_{\tilde{\epsilon}} \; \sup_{M \in \mathcal{M}}  \sum_{i=1}^N \sum_{j=1}^m \tilde{\epsilon}_{ij} M_{ij}
        \end{split}
        \label{eq: bound2_rademacher}
    \end{equation}
    where $\mathcal{M}$ is defined in~\eqref{eq: eq: calM_def}.

    \textbf{Step 3}. Next, using \textbf{Step 1}, we note that the process $(Z_M)_{M\in\mathcal{M}}$ given by
    \begin{equation}
        Z_M = \sum_{i=1}^N \sum_{j=1}^m \tilde{\epsilon}_{ij} M_{ij},
    \end{equation}
    i.e., the double sum from the right-hand-side of~\eqref{eq: bound2_rademacher}, is a Rademacher process. Thus, from~\eqref{eq: radius_Rademacher_definition}, and noting the identity
    \begin{equation}
        \begin{split}
        &\mathbb{E}_{\tilde{\epsilon}} \left|\sum_{i=1}^N \sum_{j=1}^m \tilde{\epsilon}_{ij} M_{ij} \right|^2=\mathbb{E}_{\tilde{\epsilon}} \left(\sum_{i=1}^N \sum_{j=1}^m \tilde{\epsilon}_{ij} M_{ij} \right)^2\\
        &=\mathbb{E}_{\tilde{\epsilon}} \left(\sum_{i=1}^N \sum_{j=1}^m {\tilde{\epsilon}_{ij}}^2 M_{ij}^2 + \sum_{i=1}^N \sum_{j=1}^m\sum_{\substack{k=1\\k\neq i}}^N \sum_{\substack{l=1\\l\neq j}}^m \tilde{\epsilon}_{ij} M_{ij}\tilde{\epsilon}_{kl} M_{kl}\right)\\
        &=\sum_{i=1}^N \sum_{j=1}^m  M_{ij}^2 + \sum_{i=1}^N \sum_{j=1}^m\sum_{\substack{k=1\\k\neq i}}^N \sum_{\substack{l=1\\l\neq j}}^m \underbrace{\mathbb{E}_{\tilde{\epsilon}}\tilde{\epsilon}_{ij} M_{ij}\tilde{\epsilon}_{kl} M_{kl}}_{=0},
        \end{split}
    \end{equation}
    we estimate the radius of this process as
    \begin{align}
        \Delta(\mathcal{M})
        &= \sup_{M \in \mathcal{M}} \sqrt{\mathbb{E}_{\tilde{\epsilon}} |Z_M|^2}
        \\
        &= \sup_{M \in \mathcal{M}} \sqrt{\mathbb{E}_{\tilde{\epsilon}} \left|\sum_{i=1}^N \sum_{j=1}^m \tilde{\epsilon}_{ij} M_{ij} \right|^2}
        \\
        &= \sup_{M \in \mathcal{M}} \sqrt{\sum_{i=1}^N \sum_{j=1}^m \left| M_{ij} \right|^2}
        \\
        &= \sup_{h \in \mathcal{H}} \sqrt{\sum_{i=1}^N \| h(d_i)\|^2}
        \\
        &= \sup_{\theta=(\phi,\psi) \in \Theta} \sqrt{\sum_{i=1}^N \| P_\psi(x^\star_{\psi, d_i})\|^2}
        \\
        &\leq \sqrt{N} C_{\rm{out}},
        \label{eq: radius_upper_bound}
    \end{align}
    where the last inequality is a result of Assumption~\ref{assumption: bounded_inputs_outputs}.

    \textbf{Step 4}. Starting with \eqref{eq: bound2_rademacher} and then applying Dudley's inequality (see Theorem~\ref{thm: dudley}) on the right hand side and plugging in the radius bound from~\eqref{eq: radius_upper_bound} into , we arrive at our result
    \begin{align}
        R_S(\mathcal{\ell \circ \mathcal{H}}) &\leq
        \frac{\sqrt{2} L_\ell}{N} \: \mathbb{E}_{\tilde{\epsilon}} \; \sup_{M \in \mathcal{M}}  \sum_{i=1}^N \sum_{j=1}^m \tilde{\epsilon}_{ij} M_{ij}
        \\
        &\leq  \frac{8 L_\ell}{N}\int_0^{\frac{\sqrt{N}C_{\rm{out}}}{2}} \sqrt{ \log(\mathcal{N}(\mathcal{\mathcal{M}}, \| \cdot \|, r))} dr
    \end{align}
    where $\mathcal{N}$ is the covering number defined in Definition~\ref{def: covering_number}, and the pseudometric function $\| \cdot \|$ is given by~\eqref{eq: pseudometric_is_l2} from Step 1.
\end{proof}

\subsection{Proof of Lemma~\ref{lemma: bound_covering_number}}
\noindent
\textbf{Lemma \ref{lemma: bound_covering_number}.}
\textit{\lemmaBoundCoveringNumber{app}}
\begin{proof}
Let $r>0$ and $L = \dfrac{L_{\psi}}{1 - L_x}$ from Lemma~\ref{lemma: Lipschitz_fixed_points} and $\hat L=\sqrt{((L_{P,x}L)^2+L_{P,\phi}^2)}$ from Lemma~\ref{lemma: Lipschitz_network}. Moreover, let $k:=\mathcal{N}(\Theta, \| \cdot \|, \frac{r}{\hat L})$. Then, there are $\theta_1,\ldots,\theta_k\in\Theta$ with $\theta_i=(\phi_i,\psi_i)$ such that
\begin{equation}
\label{eq:ThetaCover}
\Theta\subset\bigcup_{i=1}^kB_{\frac{r}{\hat L}}(\theta_i).
\end{equation}
Here, $B_r(x)$ denotes a ball with radius $r>0$ centered at $x$, i.e., $B_r(x):=\{y\in X:\|x-y\|_X<r\}$ for any normed vector space $(X,\|\cdot\|_X)$.
Let $y\in\mathcal{M}$. By definition of $\mathcal{M}$ (cf.~\eqref{eq: eq: calM_def}), there is a $\theta=(\phi,\psi) \in \Theta$ such that $y=P_\phi(x^\star_{\psi, d})$. By \eqref{eq:ThetaCover}, there is $j\in\{1,\ldots,k\}$ such that $\theta\in B_{\frac{r}{\hat L}}(\theta_j)$. With this we get
\begin{align}
\|y-P_{\phi_j}(x^\star_{\psi_j, D})\|=\|P_\phi(x^\star_{\psi, d})-P_{\phi_j}(x^\star_{\psi_j, D})\|\leq \hat L\|\theta-\theta_j\|\overset{\theta\in B_{\frac{r}{\hat L}}(\theta_j)}{\leq} \hat L\frac{r}{\hat L}=r.
\end{align}
Thus, we have shown $y\in B_r(P_{\phi_j}(x^\star_{\psi_j, D}))$. Since $x\in\mathcal{M}$ was arbitrary, we get
\begin{equation}
    \mathcal{M}\subset\bigcup_{i=1}^k B_r(P_{\phi_i}(x^\star_{\psi_j, D})).
\end{equation}
This immediately implies
\begin{equation}
    \label{eq:CoverBoundMByTheta}
    \mathcal{N}(\mathcal{M}, \| \cdot \|, r) \leq k = \mathcal{N}(\Theta, \| \cdot \|, \tfrac{r}{\hat L}).
\end{equation}
By Assumption~\ref{assumption: bounded_inputs_outputs}, we have, for $\theta=(\phi,\psi)\in\Theta$, that $\|\theta\|\leq \sqrt{\|\phi\|^2+\|\psi\|^2}\leq \sqrt{C_{\rm{params},\Phi}^2+C_{\rm{params},\Psi}^2}=:C_{\rm{params}}$ and thus
\begin{equation}
    \Theta = \{ \theta \in \mathbb{R}^p \; : \; \|\theta\| \leq C_{\rm{params}} \}= \{ \theta \in \mathbb{R}^p \; : \; \tfrac{1}{C_{\rm{params}}}\|\theta\| \leq 1 \}
\end{equation}
Thus, $\Theta$ is (contained in) the closed one ball of $\mathbb{R}^p$ with respect to the norm $\tfrac{1}{C_{\rm{params}}}\|\cdot\|$. Using \cite[Proposition C.3]{foucart2013mathematical}, we get for any $t>0$
\begin{equation}
    \label{eq:ThetaOneBallCoverBound}
    \mathcal{N}\left(\Theta, \tfrac{1}{C_{\rm{params}}}\|\cdot\|, t\right)\leq\left(1+\tfrac{2}{t}\right)^p.
\end{equation}
Noting that $\|\cdot\|< t$ is equivalent to $c\|\cdot\|< ct$ for any $c>0$, it holds that
\begin{equation}
\label{eq:CoverScaleInvar}
\mathcal{N}(\Theta, \|\cdot\|, t)=\mathcal{N}(\Theta, c\|\cdot\|, ct).
\end{equation}
Combining the arguments collected above results in
\begin{equation}
    \begin{split}
    \mathcal{N}(\Theta, \| \cdot \|, \tfrac{r}{\hat L})&\overset{\text{\eqref{eq:CoverScaleInvar}}}{=}\mathcal{N}\left(\Theta, \tfrac{1}{C_{\rm{params}}}\| \cdot \|, \tfrac{1}{C_{\rm{params}}}\tfrac{r}{\hat L}\right)\\&\overset{\text{\eqref{eq:ThetaOneBallCoverBound}}}\leq\left(1+\frac{2}{\tfrac{1}{C_{\rm{params}}}\tfrac{r}{\hat L}}\right)^p=\left(1+\frac{2\hat LC_{\rm{params}}}{r}\right)^p,
    \end{split}
\end{equation}
which, together with \eqref{eq:CoverBoundMByTheta}, concludes the proof.
\end{proof}

\subsection{Proof of Lemma~\ref{lemma: bound_loss}}
\noindent
\textbf{Lemma \ref{lemma: bound_loss}.}
\textit{\lemmaBoundLoss{app}}
\begin{proof}
    Since $\mathcal{D}$
    is a compact set (Assumption~\ref{assumption: compact_support}) and $\ell$ is continuous with respect to its arguments (Assumption~\ref{assumption: lipschitz_loss}), we have that $\ell$ must be bounded.
\end{proof}

\subsection{Proof of Theorem~\ref{thm: main_theorem}}
\noindent
\textbf{Theorem \ref{thm: main_theorem}.}
\textit{\thmGeneralizationBound{app}}

\begin{proof}
    Plugging the result from Lemma~\ref{lemma: bound_covering_number} into Lemma~\ref{lemma: bound_rademacher}, we have that
    \begin{align}
        R_S(\ell \circ \mathcal{H})
        &\leq \frac{8L_\ell}{N} \int_{0}^{\sqrt{N}C_{\rm{out}}/2} \sqrt{\log\left( 1 + \frac{2\hat LC_{\rm{params}}}{r} \right)^p} dr
        \\
        &= \frac{8L_\ell\sqrt{p}}{N}\int_{0}^{\sqrt{N}C_{\rm{out}}/2} \sqrt{\log\left( 1 + \frac{2\hat LC_{\rm{params}}}{r} \right)} dr. \label{eq: temp_rademacher_bound}
    \end{align}
    Using
    \begin{equation}
        \int_0^\alpha \sqrt{1 + \frac{\beta}{t}} dt\overset{\varphi(t)=\frac{\beta}{t}}=\int_0^\alpha \sqrt{1 + \frac{1}{\varphi(t)}}\beta\varphi'(t) dt=\beta\int_0^{\frac{\alpha}{\beta}} \sqrt{1 + \frac{1}{t}}dt,
    \end{equation}
    and \cite[Lemma C.9]{foucart2013mathematical}, we get that
    \begin{equation}
        \int_0^\alpha \sqrt{1 + \frac{\beta}{t}} dt \leq \alpha \sqrt{\log \left( e \cdot \left(1 + \frac{\beta}{\alpha}\right) \right)}, \quad \text{ for all } \; \alpha, \beta > 0.
    \end{equation}
    Consequently, by setting  $\beta = 2\hat LC_{\rm{params}}$ and $\alpha = \frac{\sqrt{N}C_{\rm{out}}}{2}$, we can further bound~\eqref{eq: temp_rademacher_bound} as
    \begin{align}
        R_S(\ell \circ \mathcal{H})
        &\leq \frac{8L_\ell \sqrt{p}}{N} \frac{\sqrt{N}C_{\rm{out}}}{2} \sqrt{\log\left( e \cdot \left(1 + \frac{2\hat LC_{\rm{params}}}{\sqrt{N}C_{\rm{out}}/2} \right) \right)}
        \\
        &=  \frac{4L_\ell C_{\rm{out}} \sqrt{p}}{\sqrt{N}} \sqrt{\log\left( e \cdot \left(1 + \frac{4\hat LC_{\rm{params}}}{\sqrt{N}C_{\rm{out}}} \right) \right)}.
    \end{align}
    Plugging this bound into Theorem~\ref{thm: expected_loss_error}, we obtain our desired result.
\end{proof}

\subsection{Proof of Lemma~\ref{lemma: vanilla_T_Lipschitz}}
\noindent
\textbf{Lemma \ref{lemma: vanilla_T_Lipschitz}.}
\textit{\lemmaVanillaTLipschitz{app}}

\begin{proof}
    Since $\|W\| < 1$ and $\sigma$ is 1-Lipschitz, we have contractivity guaranteed~\cite{ryu2022large}. Thus, we only need to show Lipschitzness in $\psi$ on the set of fixed points. Let $\psi_1 = (W_1, U_1, b_1) \in \Psi$ and $\psi_2 = (W_2, U_2, b_2) \in \Psi$ and $d \in \mathcal{D}$. Moreover, let $x \in \mathcal{X}^\star = \{ x \in \mathbb{R}^n \: : \: \exists \psi \in \Psi, d \in \mathcal{D} \; \text{ s.t. } \; x = T_{\psi}(x;d) \}$.
    We have
    \begin{align}
        \|T_{\psi_1}(x;d) - T_{\psi_2}(x;d)\|
        &= \| \sigma(W_1x + U_1d + b_1) - \sigma(W_2x + U_2d + b_2)\|
        \\
        &\leq \| (W_1x + U_1d + b_1) - (W_2x + U_2d + b_2)\|
        \\
        &\leq \| (W_1 - W_2) x \|+\| (U_1 - U_2) d \|+\|b_1 - b_2\|
        \\
        ~^{\text{ using }x \in \mathcal{X}^\star}_{\text{ and Assumption~\ref{assumption: bounded_inputs_outputs}}}&\leq C_{\text{out}} \| W_1 - W_2 \|+C_d \| U_1 - U_2 \|+ \|b_1 - b_2\|
        \\
        &=(C_{\text{out}},C_d,1)\cdot(\| W_1 - W_2 \|,\| U_1 - U_2 \|,\|b_1 - b_2\|)
        \\
        &\leq\sqrt{C_\text{out}^2 + C_d^2 + 1} \sqrt{\| W_1 - W_2 \|^2+\| U_1 - U_2 \|^2+\|b_1 - b_2\|^2}
        \\
        &=L_{\psi} \| \psi_1 - \psi_2 \|
    \end{align}
    where
    $
        L_{\psi} := \sqrt{C_\text{out}^2 + C_d^2 + 1}
    $
    and we have used the Cauchy-Schwarz inequality and that $\|\cdot\|$ is the Euclidean norm.
\end{proof}

\subsection{Proof of Lemma~\ref{lemma: MON_Lipschitz}}
\noindent
\textbf{Lemma \ref{lemma: MON_Lipschitz}.}
\textit{\lemmaMONLipschitzConstants{app}}
\begin{proof}
    It is well-known that $T$ is contractive in $x$~\cite{ryu2022large, winston2020monotone}.
    Thus, we only need to show Lipschitzness in $\psi$ on the set of fixed points.

    For $i=1,2$, let $\psi_i = (A_i, B_i, U_i, b_i)$ and $W_i=(1-m)I - A_i^\top A_i + B_i - B_i^\top$.
    Moreover, let $F_{W} = I - \alpha(I-W)$ and %
    $x \in \mathcal{X}^\star = \{ x \in \mathbb{R}^n \: : \: \exists \psi \in \Psi, d \in \mathcal{D} \; \text{ s.t. } \; x = T_{\psi}(x;d) \}$. Then,
    \begin{align}
        &\|T_{\psi_1}(x;d) - T_{\psi_2}(x;d) \| = \| \sigma( F_{W_1}x + \alpha(U_1d + b_1)) - \sigma( F_{W_2}x + \alpha(U_2d + b_2)) \|
        \\
        &\leq \| ( F_{W_1}x + \alpha(U_1d + b_1)) - ( F_{W_2}x + \alpha(U_2d + b_2)) \|
        \\
        &\leq \| (I - \alpha(I-W_1))x -  (I - \alpha(I-W_2))x \| + \| \alpha(U_1d + b_1) - \alpha(U_2d + b_2) \|
        \\
        &= \| \alpha (W_1 - W_2) x\| + \alpha\| (U_1-U_2)d + b_1 -b_2 \|
        \\
        &\leq \alpha (\|W_1 - W_2\|\|x\| + \| U_1 - U_2 \|\| d \| + \| b_1 -b_2 \|)\\
        &\leq \alpha (C_{\text{out}} \|W_1 - W_2\| + C_d \| U_1 - U_2 \| + \| b_1 -b_2 \|)
    \end{align}
    where the second inequality is due to 1-Lipschitzness of $\sigma$ and the last inequality arises from Assumption~\ref{assumption: bounded_inputs_outputs} and $x \in \mathcal{X}^\star$ to get $\|x\|\leq C_{\text{out}}$ and Assumption~\ref{assumption: compact_support} to get $\| d \|\leq C_d$, where $C_d = \max_d \{\|d\| : d \in \mathcal{D}  \}$.

    Recalling $W_i=(1-m)I - A_i^\top A_i + B_i - B_i^\top$, we get
    \begin{align}
        \|W_2 - W_1\| &= \| (1-m)I - A_1^\top A_1 + B_1 - B_1^\top - ((1-m)I - A_2^\top A_2 + B_2 - B_2^\top) \|
        \\
        &= \| A_2^\top A_2 - A_1^\top A_1 + B_1 - B_1^\top - B_2 + B_2^\top \|
        \\
        &\leq \| A_2^\top (A_2 - A_1) + (A_2^\top - A_1^\top) A_1 \| + \| B_1 - B_2 + B_2^\top - B_1^\top\|
        \\
        &\leq \| A_2^\top (A_2 - A_1) \| + \| (A_2^\top - A_1^\top) A_1 \|+ \| B_1 - B_2\| + \|B_2^\top - B_1^\top\|
        \\
        ~^\text{Assumption~\ref{assumption: bounded_inputs_outputs}}&\leq 2C_{\text{params}} \|A_2 - A_1\| +2\| B_2 - B_1 \|.
    \end{align}
    Combined with the above, we get
    \begin{align}
        &\|T_{\psi_1}(x;d) - T_{\psi_2}(x;d) \| \\
        &\leq \alpha (C_{\text{out}} (2C_{\text{params}} \|A_2 - A_1\| +2\| B_2 - B_1 \|) + C_d \| U_1 - U_2 \| + \| b_1 -b_2 \|)\\
        &=\alpha(2C_{\text{params}}C_{\text{out}}, 2C_{\text{out}}, C_d, 1) \cdot( \|A_2 - A_1\|, \| B_2 - B_1 \|, \| U_1 - U_2 \|, \| b_1 -b_2 \|)\\
        &\leq\alpha\sqrt{4C_{\text{params}}^2C_{\text{out}}^2+4C_{\text{out}}^2+ C_d^2+1}\sqrt{\|A_2 - A_1\|^2+\| B_2 - B_1 \|^2+\| U_1 - U_2 \|^2+\| b_1 -b_2 \|^2}\\
        &=\alpha\sqrt{4(C_{\text{params}}^2+1)C_{\text{out}}^2+ C_d^2+1}\| \psi_1 -\psi_2 \|,
    \end{align}
    where we have used the Cauchy-Schwarz inequality and that $\|\cdot\|$ is the Euclidean norm.
    Thus, we have that the Lipschitz constant of $T$ w.r.t $\psi = (A,B,U,b)$ is given by
    \begin{equation}
        L_{\psi} = \alpha\sqrt{4(C_{\text{params}}^2+1)C_{\text{out}}^2+ C_d^2+1}.
    \end{equation}
    This shows that that Assumption~\ref{assumption: T_lipschitz_contractive} is satisfied for this architecture.
\end{proof}

\subsection{Proof of Lemma~\ref{lemma: learnedGD}}
\noindent
\textbf{Lemma \ref{lemma: learnedGD}.}
\textit{\lemmaLearnedGradDescentRegularizer{app}}
\begin{proof}
    Contractivity of $T$ in $x$ is well-known when $\alpha$ is prescribed as above~\cite{ryu2022large}. Thus, we only need to show Lipschitzness of $T$ with respect to $R$ at the fixed points.

    Let $R_1, R_2 \in \mathbb{R}^{n_1 \times n}$ be given. Let $x \in \mathcal{X}^\star = \{ x \in \mathbb{R}^n \: : \: \exists \psi \in \Psi, d \in \mathcal{D} \; \text{ s.t. } \; x = T_{\psi}(x;d) \}$  belong to the set of fixed points. We have
    \begin{align}
        &\| T_{R_1}(x;d) - T_{R_2}(x;d) \|
        \\
        &= \left\| x - \alpha(A^\top (Ax - d) + R_1^\top R_1x) -
        \left( x - \alpha(A^\top (Ax - d) + R_2^\top R_2x)\right) \right\|
        \\
        &= \|\alpha (R_2^\top R_2 - R_1^\top R_1) x\|
        \\
        ~^{\text{ using }x \in \mathcal{X}^\star}_{\text{ and Assumption~\ref{assumption: bounded_inputs_outputs}}}&\leq \alpha C_{\text{out}} \| R_2^\top R_2 - R_2^\top R_1 + R_2^\top R_1 - R_1^\top R_1\|
        \\
        &\leq \alpha C_{\text{out}} \left(\| R_2^\top R_2 - R_2^\top R_1 \| + \|R_2^\top R_1 - R_1^\top R_1\|\right)
        \\
        ~^{\text{ using}}_{\text{Assumption~\ref{assumption: bounded_inputs_outputs}}}&\leq 2 \alpha C_{\text{out}} C_{\text{params},\Psi} \|R_2 - R_1\|
    \end{align}
\end{proof}

\subsection{Lipschitzness of Cross Entropy of Softmax}
\begin{lemma}
    \label{lemma: Lipschitz_CE}
    Let $s(x) = \frac{\exp(x)}{\sum_{i=1}^{n} \exp(x)_i}$ be the softmax function, and let $y \in \mathbb{R}^{n}$ be a one-hot vector. Moreover, consider function $g:\mathbb{R}^n\to\mathbb{R}, x\mapsto C(s(x),y)$, where $C$ is the cross entropy function.
    Then, then $g$ is well-defined and Lipschitz with constant 2.
\end{lemma}

\begin{proof}
    Suppose $y$ has value $1$ in the $k^{th}$ entry. Let $x\in\mathbb{R}^n$. Then, since the exponential function is never zero, we have
    \begin{equation}
        s(x)_j = \frac{\exp(x_j)}{\sum_{i=1}^n \exp(x_i)}> 0, \;\; j=1,\ldots,n.
    \end{equation}
    Thus,
    \begin{align}
        g(x) &= -y^\top \log(s(x))
        \\
        &= -\log(s(x)_k)
    \end{align}
    is well-defined.

    Next, we derive a formula for the gradient of $g$. We have
    \begin{align}
        \partial_{x_j} g(x) =-\frac{1}{s(x)_k} \partial_{x_j}(s(x)_k) = \begin{cases}
            \dfrac{\exp(x_j)}{\sum_{i=1}^n \exp(x_i)} &\text{when} \quad j \neq k
            \\
            \dfrac{\exp(x_k)}{\sum_{i=1}^n \exp(x_i)} - 1& \text{when} \quad j = k.
        \end{cases}
    \end{align}
    Thus, the gradient is given by $\nabla g(x) = s(x) - y$. Taking its norm, we estimate the Lipschitz constant of $g$ to be
    \begin{align}
        \| \nabla g(x) \|_2 = \| s(x) - y \|_2 &\leq  \underbrace{\| s(x) \|_2}_{\leq \|s(x)\|_1 = 1} + \| y \|_2
        \\
        &\leq 2.
    \end{align}
    where we note $\|y\|_2 = 1$ for one-hot vectors.
\end{proof}

\section{Robustness Experiment}
\label{sec:additions_experiments}
We show that the derived generalization bound is robust to perturbations to the parameters. In particular, we consider $\pm 10\%$ perturbations to the parameters $C_\text{params}, C_\text{out}, L_\psi,$ and $\hat{L}$ using the Contractive Single Layer architecture on the MNIST dataset. We note that one can only see three out of the nine bounds due to the similarity to one another.
\begin{figure}
    \begin{center}
    \includegraphics[width=0.6\textwidth]{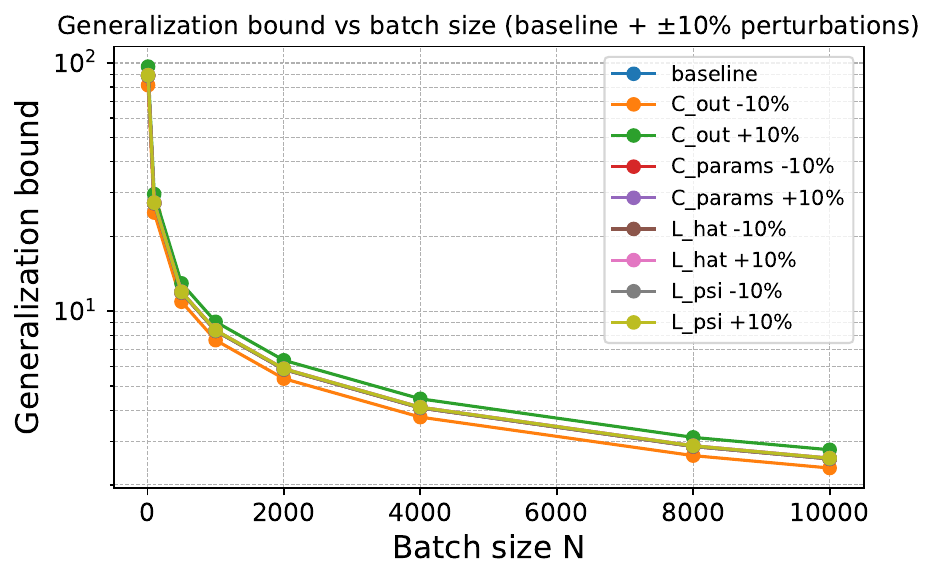}
    \end{center}
    \caption{Illustration of robustness of generalization bound under $\pm 10$ perturbations to $L_\psi, \hat{L}, C_\text{out},$ and $C_\text{params}$.}
\end{figure}

\end{document}